\documentclass[twoside]{article}

\usepackage[preprint]{aistats2026}
\usepackage{amsmath,amsfonts,bm}

\def\eqref#1{equation~\ref{#1}}

\def\1{\bm{1}}

\DeclareMathAlphabet{\mathsfit}{\encodingdefault}{\sfdefault}{m}{sl}
\SetMathAlphabet{\mathsfit}{bold}{\encodingdefault}{\sfdefault}{bx}{n}

\newcommand{\R}{\mathbb{R}}

\DeclareMathOperator*{\argmax}{arg\,max}

\usepackage{hyperref}
\usepackage{url}
\usepackage{enumitem}
\usepackage{algorithm}
\usepackage{algorithmic}
\usepackage{amssymb}
\usepackage{bbm}
\usepackage{bm}
\usepackage{graphicx}
\usepackage{amsmath,amsthm}
\usepackage{natbib}
\usepackage[dvipsnames]{xcolor}
\usepackage{caption}
\usepackage{subcaption}
\usepackage{booktabs}
\usepackage{tabularx}
\usepackage{multirow}
\usepackage{booktabs}
\usepackage{makecell} 

\usepackage{color-edits}
\addauthor[Elisabeth]{ep}{blue}
\addauthor[Kirk]{kb}{green}
\addauthor[Paula]{prd}{orange}

\newcommand{\ip}{\mathrm{IP}}
\newcommand{\lp}{\mathrm{LP}}
\newcommand{\by}{\mathbf{y}}
\newcommand{\bx}{\mathbf{x}}
\newcommand{\bz}{\mathbf{z}}
\newcommand{\bb}{\mathbf{b}}
\newcommand{\bl}{\boldsymbol{\lambda}}
\newcommand{\bA}{\mathbf{A}}
\newcommand{\bh}{\mathbf{h}}

\newtheorem{theorem}{Theorem}
\newtheorem{lemma}{Lemma}

\begin{document}

%

%
\runningtitle{A Dual Perspective on Decision-Focused Learning}

\twocolumn[

\aistatstitle{A Dual Perspective on Decision-Focused Learning:\\ Scalable Training via Dual-Guided Surrogates}

\aistatsauthor{ Paula Rodriguez-Diaz \And Kirk Bansak \And  Elisabeth Paulson }

\aistatsaddress{ Harvard University \And  University of California, Berkeley \And Harvard University } 
]

\begin{abstract}
Many real-world decisions are made under uncertainty by solving optimization problems using predicted quantities. This predict-then-optimize paradigm has motivated \textit{decision-focused learning}, which trains models with awareness of how the optimizer uses predictions, improving the performance of downstream decisions. Despite its promise, scaling is challenging: state-of-the-art methods either differentiate through a solver or rely on task-specific surrogates, both of which require frequent and expensive calls to an optimizer, often a combinatorial one. In this paper, we leverage \textit{dual variables} from the downstream problem to shape learning and introduce \textit{Dual-Guided Loss} (DGL), a simple, scalable objective that preserves decision alignment while reducing solver dependence. We construct DGL specifically for combinatorial selection problems with natural one-of-many constraints, such as matching, knapsack, and shortest path. Our approach (a) decouples optimization from gradient updates by solving the downstream problem only periodically; (b) between refreshes, trains on dual-adjusted targets using simple differentiable surrogate losses; and (c) as refreshes become less frequent, drives training cost toward standard supervised learning while retaining strong decision alignment. We prove that DGL has asymptotically diminishing decision regret, analyze runtime complexity, and show on two problem classes that DGL matches or exceeds state-of-the-art DFL methods while using far fewer solver calls and substantially less training time. Code is available at \url{https://github.com/paularodr/Dual-Guided-Learning}.\looseness=-1
\end{abstract}

\section{Introduction}
In many real-world decision-making settings, decisions are made under uncertainty and then used to inform downstream actions. The Predict-then-Optimize (PtO) framework formalizes this workflow, with predictions fed into a downstream optimization solver that returns decisions \citep{elmachtoub_smart_2022}. The prevailing approach is a two-stage pipeline (Fig \ref{fig:training_pipelines}a): (i) train a predictor to minimize a statistical loss (e.g., mean-squared error or negative log-likelihood), then (ii) plug those predictions into an optimization solver to pick actions. In this setup, model training optimizes for predictive accuracy, not decision quality, and the two objectives can diverge. For example, when minimizing prediction accuracy, errors are weighted uniformly rather than by their impact on feasibility and cost; small errors near decision boundaries can flip the chosen action and incur large decision regret. In general, problem-specific constraints and asymmetric costs are ignored under a two-stage approach.\looseness=-1

To address this mismatch, decision-focused learning (DFL) methods train the predictor inside this PtO pipeline to prioritize downstream decision quality over prediction accuracy alone \citep{mandi_decision-focused_2024}. DFL is typically achieved through two routes. One embeds a \emph{differentiable optimization layer} so the model is trained end-to-end through the optimization solver (e.g., OptNet; CVXPYLayers), which requires a forward solve and backpropagation through the optimization problem at each update \citep{wilder_melding_2019, amos_optnet_2017} (Fig \ref{fig:training_pipelines}b). The other replaces direct regret minimization with \emph{surrogate objectives} that approximate decision loss but still invoke one or more solves per batch/iteration (often on loss-augmented or smoothed problems) \citep{elmachtoub_smart_2022,mulamba_contrastive_2021}. These approaches have been shown to make higher-quality decisions on held-out test data in a variety of problem settings compared to the standard PtO \citep{mandi_decision-focused_2024}.
However, in both cases, optimization is invoked repeatedly throughout training, validation, and tuning, which creates substantial computational and memory overhead \citep{mandi_smart_2020, shah_decision-focused_2022, mulamba_contrastive_2021}. These costs are especially acute for combinatorial tasks—assignments, matchings, knapsack, packing, routing—where exact solves are expensive and even relaxed or approximate oracle calls remain costly. Consequently, runtime and scalability are central obstacles for deploying DFL on realistic data and problem sizes \citep{wang_scalable_2023}.\looseness=-1

In this work we address the tension between DFL and efficiency. Our key idea is that information from the dual formulation of the downstream optimization problem, obtained from a single solve, can guide many subsequent gradient steps without re-invoking the solver. Concretely, we refresh dual variables (Lagrange multipliers) on predictions for the current model only periodically, and between refreshes we train on dual-adjusted targets using a simple surrogate loss. This approach decouples the frequency of optimization from the update loop, reducing overhead while preserving decision awareness (Fig \ref{fig:training_pipelines}c). We show that dual-guided learning dramatically improves scalability and runtime efficiency, especially in combinatorial settings, while maintaining competitive task performance. By addressing the core bottleneck of repeated solves, our framework takes a step toward making decision-focused methods practical for larger-scale and time-sensitive applications. In summary, we make the following contributions:

\begin{itemize}[leftmargin=0.5cm]
    \item \textbf{Dual-Guided Loss (DGL)}. We introduce a new decision-focused learning framework that leverages dual variables (Lagrange multipliers) to guide training. DGL requires only periodic solver calls, decoupling solve frequency from gradient updates. DGL applies to a specific class of decision problems, including the assignment and knapsack, and shortest path problems.
    \item \textbf{Substantial efficiency gains}. Compared to state-of-the-art DFL methods, we demonstrate that DGL dramatically reduces runtime, making training practical in combinatorial optimization settings where repeated solves are prohibitive.
    \item \textbf{Decision quality}. We show empirically that DGL achieves competitive or superior task performance compared to the highest-performance solver-in-the-loop baselines, highlighting that reduced computation does not come at the expense of decision quality.\looseness=-1 
\end{itemize}

\section{Related Work}\label{sec:literature}

\paragraph{Decision-Focused Learning.} 

DFL trains predictors for downstream decisions by integrating an optimization oracle during training. State-of-the-art methods fall into two families \cite{mandi_decision-focused_2024}. (i) Differentiable optimization layers \cite{wilder_melding_2019, amos_optnet_2017} (e.g., OptNet, CVXPYLayers) embed a solver in the forward pass and use implicit differentiation through KKT systems or factorized solvers; each step solves (or partially solves) the problem, creating high overhead on large or combinatorial tasks. Following \cite{mandi_decision-focused_2024}, we call this Quadratic Programming Task Loss (QPTL). (ii) Surrogate-loss methods avoid differentiating the solver by using decision-aware losses with. SPO+ is a standard example \citep{elmachtoub_smart_2022}. It still requires solving the optimization for each training sample. In both families the solver remains in the loop and becomes the bottleneck at scale.

Subsequent work aims to reducing computational overhead along three lines. First, as a preliminary but limited method, LP relaxations with warm starts lower per-iteration cost while preserving decision alignment on hard discrete tasks \citep{mandi_smart_2020}. These speed-ups still require solving an optimization problem throughout training, and thus the resulting training times can still be prohibitive. Second, solution caching \citep{mulamba_contrastive_2021} wraps SPO+ or QPTL, calling the solver intermittently and reusing cached primal solutions to form gradients, which yields large speedups at the cost of approximation error and memory. Third, locally optimized decision losses (LODLs) \citep{shah_decision-focused_2022} move compute outside the update loop by learning instance-specific surrogate losses from oracle decisions. LODL is broadly applicable, including nonconvex cases, but fidelity depends on the learned loss class and the coverage of the sampling procedure.\looseness=-1

Conceptually, caching speeds up existing objectives by reusing cached solutions and LODL replaces them with a learned surrogate. By contrast, DGL directly injects periodically refreshed dual variables into a fixed loss function, so for problems with informative duals it stays tightly aligned with the true optimization objective and typically yields better decision quality. Practically, caching and LODL favor generality across problem classes, whereas DGL trades breadth for a lightweight, scalable design that aligns closely with the downstream optimization problem for a particular (but common) class of problems.

\begin{figure*}
    \centering
    \includegraphics[trim=0 0 0 0,clip,width=\linewidth]{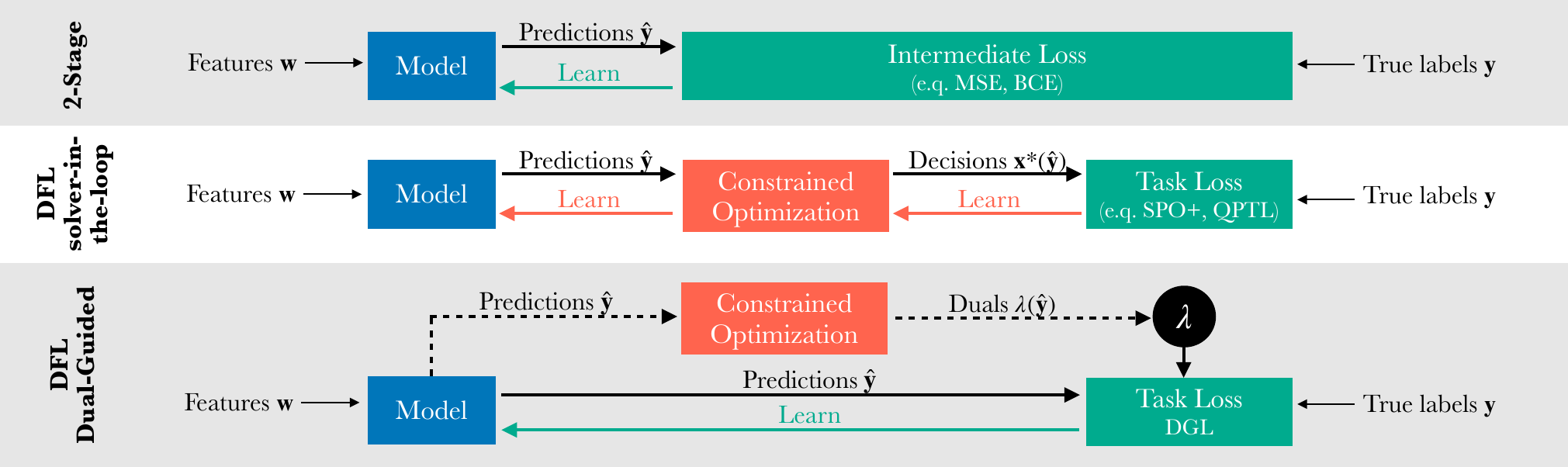}
    \caption{\textbf{Training Pipelines.} (a) \textit{Two-stage}: Training loss ignores decision quality. (b) \textit{DFL with solver-in-the-loop}: train through an optimization oracle—either by differentiating through it (QPTL) or using solution-based gradients (SPO)—incurring a solve at every training step. (c) \textit{DFL with Dual-Guided Loss}: periodically refresh duals $\mathbf{\lambda}$ from the downstream problem and, between refreshes, train on dual-adjusted soft decisions.}
    \label{fig:training_pipelines}
\end{figure*}

\paragraph{Dual Variables in Constrained Training.}
Work outside DFL uses Lagrange dual variables to enforce structure or constraints on predictions. Representative lines include augmented-Lagrangian and primal–dual training for hard/structural constraints \citep{fioretto_lagrangian_2020,nandwani_primal-dual_2019}, structured prediction via posterior regularization and constrained conditional models \citep{ganchev_posterior_nodate,chang_structured_2012}, fairness reductions and rate-constrained methods \citep{agarwal_reductions_2018,cotter_training_2019}, and theory-guided networks that encode physics laws as constraints \citep{rong_lagrangian_2022}. These approaches use duals to enforce validity; in contrast, we use duals from the downstream optimization to shape a decision-aware surrogate that steers predictions toward better end-task decisions.

\paragraph{Dual Variables in Online Optimization.} The idea of using dual variables to guide decision-making is not new. Dual variables are known to act as dynamic \textit{prices} that guide sequential decisions under uncertainty. For example, in online packing, bid-price schemes estimate these prices from samples or LP relaxations, accept a request when the marginal value exceeds the price, and refresh as inventory changes. As another example, in ad allocation, primal-dual algorithms update the prices at each arrival to track evolving traffic and capacity \citep{agrawal_dynamic_2014}. In communication networks, network utility maximization treats congestion signals as dual prices updated from queue feedback, and drift-plus-penalty uses virtual queues as Lagrange multipliers to balance throughput and constraints without a demand model \citep{kelly_rate_1998}. The common pattern is continual estimation or tracking of shadow prices so that decision rules stay aligned with the current operating point. Our dual-guided training adopts this principle inside learning: duals computed from ground truth or early predictions can become stale as the model shifts toward decision-optimal regions, so we periodically recompute multipliers and train on dual-adjusted targets, keeping the training signal consistent with the downstream problem while avoiding solves at every step.

\section{Problem Description}
\subsection{Learning Task}

DFL couples prediction with optimization by training a predictor to minimize a downstream, decision-based loss. We consider a dataset $$\mathcal{D}=\{(\bm{w}^{(k)},\by^{(k)})\}_{k=1}^K$$ consisting of $K$ optimization instances, each with $N$ choices. Each choice $i\in 1,\dots, N$ in instance $k$ is endowed with a feature vector $\bm{w}_i^{(k)}$ and an outcome $y_i^{(k)}$. Thus, 
$\bm{w}^{(k)}=(\bm{w}_i^{(k)})_{i=1}^N \text{ and } \by^{(k)}=(y_i^{(k)})_{i=1}^N.$

A predictive model $M_\theta$ maps features to parameter predictions item-wise, $\hat y_i^{(k)} = M_\theta\!\big(\bm{w}_i^{(k)}\big).$ For brevity, we use $\hat{\by}^{(k)}(\theta)$ to denote the vector of predicted parameters for instance $k$ obtained by applying $M_\theta$ to $\bm{w}^{(k)}$, $\hat{\by}^{(k)}(\theta)=\big(\hat y_i^{(k)}\big)_{i=1}^N$.

For each instance $k$, the decision is obtained by solving
\begin{equation*}
\bx_\theta^{(k)} \in \argmax_{\bx\in\mathcal{F}^{(k)}} f\!\big(\bx,\hat{\by}^{(k)}(\theta)\big),
\end{equation*}
where $\mathcal{F}^{(k)}$ is known and uncertainty enters through $\hat{\by}^{(k)}(\theta)$. Let
$
\bx^*(\by^{(k)}) \in \argmax_{\bx\in\mathcal{F}^{(k)}} f\!\big(\bx,\by^{(k)}\big)
$
denote the optimal decision under the true parameters. The per-instance \emph{decision regret} is
$$
r^{(k)}(\theta)\;=\; f\!\big(\bx^*(\by^{(k)}),\,\by^{(k)}\big)\;-\; f\!\big(\bx_\theta^{(k)},\,\by^{(k)}\big)\;\ge 0,
$$
and aggregating over instances yields the empirical average regret
$$
\widehat{\mathrm{Regret}}(\theta)\;=\;\frac{1}{K}\sum_{k=1}^K r^{(k)}(\theta).
$$
For minimization problems the signs reverse; we adopt the maximization form without loss of generality.

Directly minimizing $\widehat{\mathrm{Regret}}(\theta)$ requires optimizing through the $\argmax$, which is nondifferentiable and expensive to handle. Because small changes in the predictions can flip the discrete decision, the  loss function is flat almost everywhere with sharp points of discontinuity, leaving gradients zero or undefined. To address this, prior DFL work uses differentiable surrogates or reformulates the downstream problem as a differentiable quadratic program for backpropagation, but current strategies require frequent and costly solver calls (see Section \ref{sec:literature}).

\subsection{Downstream Problems Considered}
In this paper, we limit our focus to a particular class of downstream decision problems $f$. Specifically, we consider a family of constrained optimization problems with natural \textit{pick-one groups} where the objective is linear in a parameter vector $\mathbf{y} \in \mathbb{R}^N$, which is typically unknown at decision time. The decision variables $\mathbf{x}$ may be continuous or binary, and the problem takes the general form:
\begin{equation} \label{eq:original_LP}
\begin{aligned}
\max&_{\bx\in\{0,1\}^N}\;\,\by^\top \bx \\[4pt]
\text{s.t.}\;&\bA \bx\le \bb &&\textit{global}\\
            &\sum_{i\in g}x_i = 1 \quad\forall g\in\mathcal{G}
                       &&\textit{one-of-many}
\end{aligned}
\end{equation}

Here, $\bA \in \mathbb{R}^{M \times N}$, $b \in \mathbb{R}^M$, and  $\mathcal{G}$ partitions $\{1,\dots,N\}$ into mutually-exclusive groups. Many common decision problems can be written in this form. Notably, five of the seven benchmark tasks in the DFL survey by \cite{mandi_decision-focused_2024} can be expressed in this form. In the Appendix, we detail how the assignment and knapsack problem, used for our empirical evaluation in Section~\ref{sec:experiments}, can be cast in the form of Problem~\ref{eq:original_LP}. Beyond these two, several common decision problems also admit formulations of Problem~\ref{eq:original_LP}, including shortest path (on a grid or general graph), minimum-cost flow /transportation, and discrete optimal transport with unit-mass marginals. \looseness=-1

\section{Method} \label{sec:method}

In this section we introduce a dual-guided surrogate aligned with decision quality for tasks whose downstream problem can be written as Problem \ref{eq:original_LP}. Throughout this section we focus our attention on \emph{one} optimization instance and drop the superscript $k$. We provide an asymptotic regret bound for our loss function, present three gradient-based training variants that use this loss, and provide a time-complexity analysis relative to state-of-the-art DFL baselines.

\subsection{Dual-Guided Loss}
Let $\bl\in \mathbb{R}_+^M$ be the Lagrange multipliers for the global constraint $\bA\bx\le \bb$ in the continuous relaxation of Problem \ref{eq:original_LP}. The corresponding Lagrangian is
\begin{equation*}
L(\bx,\bl)
    =\by^\top \bx+\bl^\top(\bb-\bA\bx)
    =\bigl(\by-\bA^\top\bl\bigr)^\top \bx+\bl^\top \bb.
\end{equation*}
Because the problem includes the \emph{one-of-many} constraints, a natural candidate solution to the relaxed problem
\begin{equation}
\max_{\bx\in\{0,1\}^N}\;\,L(\bx,\bl) \quad
\text{s.t.}\; \sum_{i\in g}x_i = 1 \quad\forall g\in\mathcal{G}
\end{equation}
is given by $x_{j}=1$ if and only if $j=\arg\max_i y_{i}-(\bA^\top\bl)_{i}$ for $i\in g, \; g\in \mathcal{G}$ (ties broken arbitrarily).\looseness=-1

This motivates defining the \textit{dual-adjusted scores}, 
$h_i(\mathbf{y}, \bl)\triangleq y_i-(\bA^\top\bl)_i.$ 
To make the construction of this candidate solution differentiable, we apply a temperature-controlled softmax within each group $g$:\looseness=-1
\begin{align*}
\tilde \bz_{g,\tau}(\mathbf{y}, \bl)
  &\triangleq \text{softmax}((h_i(\mathbf{y}, \bl))_{i\in g}, \tau)\\
  &= \left(\frac{\exp(h_j(\mathbf{y}, \bl)/\tau)}
        {\sum_{i\in g}\exp(h_i(\mathbf{y}, \bl)/\tau)}\right)_{j\in g},
\end{align*}
which converges to $\argmax$ as $\tau \to 0$.
For comparison, let
$z^*_g(\mathbf{y}, \bl) = \mathbf{e}_{\argmax (h_i(\mathbf{y}, \bl))_{i\in g}}$
denote the discrete $\argmax$ assignment for group $g$, where $\mathbf{e}_i$ denotes the $i$th standard basis vector (the vector with a $1$ in the $i$th position and zeros elsewhere). Finally, we define the \textit{Dual-Guided Loss} (DGL) as 

\begin{equation}
\label{eq:dgl}
    \ell_{\tau}(\theta, \bl)\triangleq -\frac{1}{|\mathcal{G}|} \sum_{g\in\mathcal{G}}\bigl(\mathbf{y}_i\bigr)_{i\in g}^{\top} \tilde \bz_{g,\tau}(\hat{\mathbf{y}}(\theta), \bl).
\end{equation}

This loss approximates the objective of Problem \ref{eq:original_LP} using softmax assignments instead of a hard $\argmax$. Most importantly, the loss is fully differentiable and therefore avoids both repeated solves and backpropagation through the optimization problem.

\subsection{Regret Guarantees} \label{sec:regret}

Let $x^{\ip}(\by)$ and $x^{\lp}(\by)$ denote the solutions to Problem~\ref{eq:original_LP} and its LP relaxation, respectively, and let $\tilde \bz_{g,\tau}(\by,\bl)$ and $\bz_g^*(\by,\bl)$ be as defined above. Our analysis relies on three assumptions:

\begin{enumerate}
[label=\textbf{A\arabic*}.,itemsep=3pt]
\item \textbf{Integral polytope.}  
      The feasible region of the LP relaxation of Problem \ref{eq:original_LP}  is integral,
      hence its optimal solution $x^{\lp}(\by)$ is always binary and
      equals the integer optimum $x^{\ip}(\by)$ for all $\by$.
\item \textbf{Unique group maximizer.}  
 There exists $\gamma>0$ such that for every group $g\in\mathcal G$ and (i) with $c=\by$ and (ii) with $c=\hat \by(\theta_k)$ along the training sequence $\{\theta_k\}$, if $\bl$ is a dual-optimal vector then 
$h_{j_g^\star}(c, \bl)\;-\;\max_{i\in g\setminus\{j_g^\star\}} h_i(c,\bl)
\;\ge\;\gamma$ almost surely, where $j^*_g=\argmax_{i\in g} h_i(c,\bl)$.
\item \textbf{Bounded rewards.}  $\|\by\|_{\infty}\le C<\infty$ almost surely.
\end{enumerate}

Assumption A1 is a strong assumption, limiting the scope of the theoretical analysis to a particular subset of problems. The assignment problem with unit demand per worker and the shortest path problem typically satisfy this assumption. Without A1, an additional error term appears in the regret guarantee of Theorem \ref{thm:regret} that captures the potential gap between the integer optimum and the fractional optimum of the LP relaxation, i.e., the integrality gap. Assumptions A2 and A3 are regularity assumptions that likely hold in most contexts of interest.\looseness=-1

Define expected losses as 
\begin{align*}
&\mathcal L_{\tau}(\theta, \hat{\bl}):=\mathbb E[\ell_{\tau}(\theta, \hat{\bl})].
\end{align*}

where $\hat\bl$ is any dual-optimal vector for $\hat{\mathbf{y}}(\theta)$. Let $\mathcal L_{\tau}^{*}:=\inf_{\theta}\mathcal L_{\tau}(\theta)$ denote the Bayes risk. 

We can write the decision regret as: 
$$ \mathrm{Regret}(\theta) :=\mathbb{E}\left[\mathbf{y}^{\top}x^{\ip}(\mathbf{y}) - \mathbf{y}^{\top}x^{\ip}(\hat \by(\theta))\right].$$

The following result shows that minimizing DGL during training yields predictors with low decision regret; the proof is deferred to the Appendix.
\medskip
\begin{theorem}
\label{thm:regret}
Assume \textup{(A1)--(A3)}. Then for every $\tau>0$ and $\theta$,
$\mathrm{Regret}(\theta)
\;\le\;
|\mathcal G|\,\bigl[\mathcal L_{\tau}(\theta, \hat\bl)-\mathcal L_{\tau}^{*}\bigr]
\;+\;O\!\bigl(e^{-\gamma/\tau}\bigr),$
where the $O\!\bigl(e^{-\gamma/\tau}\bigr)$ term is uniform in $\theta$ (its constant may depend on $C$ and $|\mathcal G|$).
Consequently, if $\mathcal L_{\tau_k}(\theta_k, \hat{\bl}_k)-\mathcal L_{\tau_k}^{*}\to0$ and $\tau_k\downarrow0$, then $\mathrm{Regret}(\theta_k)\to0$.
\end{theorem}
We note that since $e^{-\gamma/\tau}=o(\tau)$, the remainder term is $O(\tau)$ for sufficiently small $\tau$.

\subsection{Learning with DGL loss}

Our method trains $M_\theta$ by minimizing the DGL loss with gradient descent. For each feature vector $\bm{w}_i$, the model predicts $\hat y_i = M_\theta(\bm{w}_i)$, and we form soft surrogate decisions $\mathbf{\tilde{z}}_\tau(\hat{\mathbf{y}}, \hat\bl)$, enabling full backpropagation through the differentiable surrogate loss function $\ell_{\tau}(\theta, \hat\bl)$. 

\subsubsection{Dual Refresh Strategies} \label{sec:dual_refresh}
A key challenge is setting the dual multipliers. Throughout training, $\hat{\mathbf{y}}(\theta)$ is updated by $\theta$. Theorem \ref{thm:regret} shows that if DGL loss is minimized---where $\hat{\bl}$ are dual-optimal for $\hat{\mathbf{y}}(\theta)$---then decision regret vanishes. This suggests that updating the dual multipliers throughout training can be beneficial. In practice, though, it is computationally costly (and, we find, also unnecessary) to update the dual values at every training epoch, as estimating optimal the dual multipliers requires solving the original optimization problem.
Thus, we explore three strategies that periodically refresh duals during training:\looseness=-1

\begin{itemize}[leftmargin=0cm]
    \item[] \textbf{No-update DGL:} Dual variables are computed once at initialization by solving the original optimization problem (Problem \ref{eq:original_LP}) using the ground-truth cost vectors $\mathbf{y}$. These duals are fixed throughout training.

    \item[] \textbf{Fixed-frequency DGL:} Duals are updated for all training instances every $U$ epochs using the current model predictions $\hat{\mathbf{y}}(\theta)$. Each dual update requires solving the optimization problem for every training instance with the predicted cost vector.

    \item[] \textbf{Auto-update:} Duals are updated only for instances where surrogate decision (the argmax of $\hat{\mathbf{y}} - \mathbf{A}^\top \hat\bl$) leads to constraint violation, targeting updating cases where the current duals are misleading.\looseness=-1 
\end{itemize}

\begin{algorithm}[t]
\small
\caption{DGL with Fixed-Frequency Updates}
\label{alg:dgl-fixed}
\begin{algorithmic}[1]
\REQUIRE Dataset $\{(\mathbf{w}^{(k)}, \mathbf{y}^{(k)})\}_{k=1}^K$, model $M_\theta$, temperature $\tau$, learning rate $\eta$, number of epochs $T$, dual update frequency $U$
\STATE Initialize model parameters $\theta$
\FOR{$t = 0$ to $T-1$}
    \IF{$t \bmod U = 0$}
        \FOR{each instance $(\mathbf{w}^{(k)}, \mathbf{y}^{(k)})$ in dataset}
            \STATE Compute predictions $\hat{\mathbf{y}}^{(k)} \leftarrow (M_\theta(\mathbf{w}_{i}^{(k)}))_{i=1}^{N}$
            \STATE Solve Problem~\ref{eq:original_LP} with current predictions to obtain updated duals
            $\bl_k\leftarrow \bl^*(\hat{\mathbf{y}}^{(k)})$
        \ENDFOR
    \ENDIF
    \FOR{each mini-batch $\mathcal{B} \subset \{1, \dots, K\}$}
        \STATE Compute predictions 
        $\hat{\mathbf{y}}^{(k)} \leftarrow (M_\theta(\mathbf{w}_{i}^{(k)}))_{i\in \mathcal{B}}$
        \STATE Compute surrogate decisions $\tilde{\mathbf{z}}_k = \text{softmax}(\tau \cdot (\hat{\mathbf{y}}^{(k)} - \bA^\top \bl_k))$ for each $k\in \mathcal{B}$
        \STATE Compute loss $\ell_\tau(\theta, \bl) = \frac{1}{|\mathcal{B}|} \sum_{k \in \mathcal{B}} -\mathbf{y}^{(k)\top} \tilde{\mathbf{z}}_k$
        \STATE Update model: $\theta \leftarrow \theta - \eta \cdot \nabla_\theta \ell_\tau(\theta,\bl)$
    \ENDFOR
\ENDFOR
\end{algorithmic}
\end{algorithm}

All DGL variants offer substantial computational savings over standard decision-focused learning, as they avoid solving or differentiating through the optimization layer on every gradient step. Algorithm~\ref{alg:dgl-fixed} presents the training loop under gradient descent with fixed-frequency dual refresh. For completeness, pseudocode for the no-update and auto-update dual refresh strategies is provided in Appendix.

\subsubsection{Dual-Adjusted Loss for Stable Training}
$\ell_{\tau}(\theta,\bl)$ is the natural loss function and leads to a clean regret guarantee (see Section \ref{sec:regret}). However, in practice we find that this loss function can be difficult to minimize, particularly in problem settings with tight capacity constraints. Intuitively, during training the decision distribution $\tilde{\bz}$ (defined using the softmax of $\hat{\by}-\bA^{\top}\hat{\bl}$ for group $g$) tries to maximize reward by putting more weight on high-reward items. This creates a misalignment that artificially inflates the predictions for these items, which is only partially mitigated by re-estimating the duals. Computationally, we find that training is more stable under the dual-adjusted loss\looseness=-1
\begin{equation}\small
\label{eq:dgl-dual-adjusted}
    \tilde{\ell}_{\tau}(\theta, \hat\bl)\triangleq -\frac{1}{|\mathcal G|}\sum_{g\in\mathcal G}
\bigl(\mathbf{y}_i-(\bA^{\top}\hat{\bl})_i \bigr)_{i\in g}^{\!\top}\,\tilde \bz_{g,\tau}\bigl(\hat \by(\theta),\hat\bl\bigr).
\end{equation}
This loss function evaluates predictions on the same reduced-cost scale used by softmax. This dual-adjusted variant can be interpreted as a practical stabilization device, similar to reward shaping in reinforcement learning and margin rescaling in supervised machine learning.

To further stabilize training and allow flexibility to control the tradeoff between prediction accuracy and decision loss, we also add a mean-squared error term. Thus, the loss function we use in practice is
$$\tilde \ell^{\text{MSE}}_{\tau}(\theta,\hat\bl)\triangleq \tilde \ell_{\tau}(\theta,\hat\bl)+ \alpha \frac{1}{N}\sum_{i=1}^N (y_i-\hat{y}_i)^2,$$
where $\alpha$ is a tunable hyperparameter. In the Appendix, we show an analogous result to Theorem \ref{thm:regret} for this loss function. Namely, we show that minimizing $\tilde \ell^{\text{MSE}}_{\tau}(\theta,\hat\bl)$ with $\alpha>0$ results in vanishing decision regret.\looseness=-1

\section{Time Complexity Analysis}

We summarize the state-of-the-art DFL baselines and their per-epoch complexity under gradient-based training. With $K$ gradient steps per epoch, per-instance times are: $T_M$ for a forward--backward pass of $M_\theta$, $T_O$ to solve the downstream optimization, and $T'_O$ to differentiate through the optimization layer. When an offline surrogate is used, $S$ is the sample count and $T_{\text{LODL}}$ its training time.

\begin{itemize}[leftmargin=0cm]
    \item[] \textbf{Two-Stage} optimizes the predictive model using a supervised loss (e.g., MSE) without solving or differentiating through the downstream optimization problem. Its per-epoch time complexity is $\Theta(K \cdot T_M)$.

    \item[] \textbf{SPO+} \citep{elmachtoub_smart_2022} builds a piecewise-linear surrogate from the problem’s solution, requiring one solve per instance but no differentiation through the optimization layer. A subgradient is computed via a closed-form expression using optimal and perturbed solutions. The per-epoch time complexity is $\Theta(K \cdot (T_M + T_O))$, with $T_O$ the time to solve an LP. With solution caching \citep{mulamba_contrastive_2021}, which reuses a stored solution when it remains optimal for the perturbed cost vector, the per-epoch time is $\Theta(K \cdot T_M + (1-\rho)\, K \cdot T_O)$, with cache hit rate $\rho \in [0,1]$.\looseness=-1

    \item[] \textbf{QPTL} uses a differentiable convex optimization layer, typically implemented with tools like CVXPYLayer \citep{wilder_melding_2019,amos_optnet_2017}. It augments the original linear objective with a quadratic regularization term, resulting in a QP. For each batch, it requires solving this QP during the forward pass and computing gradients through the solver using implicit differentiation. The per-epoch time complexity is $\Theta(K \cdot (T_M + T_O + T'_O))$, with $T_O$ the time to solve a QP. In practice, $T'_O$ is often the dominant term due to solver overhead and autograd computation, especially when the problem has many constraints or decision variables.

    \item[] \textbf{LODLs.} learn an offline regret surrogate from $K$ samples and solver outputs, so $M_\theta$ can train without solving the downstream problem \citep{shah_decision-focused_2022}. Per-epoch cost is $\Theta(K T_M)$ (as in two-stage); the one-time offline cost is $\Theta(S K T_O + S K T_{\text{LODL}})$, amortizing solver calls but adding upfront cost and approximation error.\looseness=-1

    \item[] \textbf{DGL (Ours)} uses fixed or periodically updated dual variables to guide training via a surrogate loss that avoids both solving and differentiating through the optimization problem during most epochs. The gradient is computed through a soft surrogate (e.g., softmax-based) that approximates the argmax over decisions, and does not require solving the optimization problem or computing gradients through a solver. If duals are updated every $U$ epochs (where $U \geq 1$), the method only solves the optimization problem once every $U$ epochs per instance. The per-epoch time complexity is therefore $\Theta(K \cdot T_M + \frac{1}{U} \cdot K \cdot T_O)$. When $U = \infty$ (i.e., duals are computed once at initialization and never updated), DGL approaches the runtime of the two-stage method while still benefiting from decision-aware guidance. As $U$ decreases, performance may improve at the cost of additional solver calls.

\end{itemize}

\paragraph{Scaling to Large Problem Instances.} While the expressions above describe per-epoch runtime, solver costs $T_O$ and $T'_O$ grow with the size of the optimization problem (e.g., the number of individuals and locations in matching). This growth disproportionately hurts QPTL, since both the forward solve and the differentiation step scale poorly. In contrast, DGL avoids solving during most epochs and is therefore significantly more scalable in practice. For large instances, this results in markedly faster training times with minimal performance degradation. We note that solution caching \citep{mulamba_contrastive_2021} can decrease the run-time of both QPTL and SPO+, though this comes at a cost in performance.

To summarize, when solver time $T_O$ dominates model time $T_M$ (typical for combinatorial tasks), DGL with refresh period $U$ reduces the amortized solver cost by $\approx 1/U$ relative to SPO+, and removes the differentiation term $T'_O$ required by QPTL. Concretely, with $K$ gradient steps per epoch,
$$
\textbf{DGL}=\Theta\!\big(K(T_M+\tfrac{1}{U}T_O)\big)$$
$$\textbf{SPO+}=\Theta\!\big(K(T_M+T_O)\big)$$
$$\textbf{QPTL}=\Theta\!\big(K(T_M+T_O+T'_O)\big).
$$
Thus, for any $U>1$, DGL is faster per epoch than SPO+, with speedup $\approx U$ when $T_O\gg T_M$. Compared to QPTL, DGL avoids $T'_O$ and cuts solver calls by $U$. Because LP/QP solves scale super-linearly in problem size, the gap widens with instance size, matching the runtime trends in our empirical evaluations (Fig.~\ref{fig:experiments}).

\section{Experiments}\label{sec:experiments}

We evaluate DGL on two DFL tasks: a synthetic matching problem motivated by algorithmic refugee assignment \citep{bansak_outcome-driven_2024} that is new to the DFL literature, and a knapsack problem with weighted items from the DFL literature \citep{mandi_decision-focused_2024}. We compare against strong baselines (Two-Stage, SPO+, and QPTL), using model architectures from prior work and identical optimization solvers across methods. We provide further details on the training procedure and hyperparameter tuning in the Appendix.\looseness=-1

We exclude LODLs to keep evaluation focused on solver in-the-loop training. A fair comparison to LODL would require training an offline surrogate on optimization-labeled data with tuning/validation, adding substantial engineering and upfront runtime that confounds per-epoch timing and is orthogonal to our contribution. We also do not evaluate cached variants of SPO+ or QPTL: while solution caching can reduce wall-clock time, prior work indicates it typically trades off decision quality \citep{mandi_decision-focused_2024}. Our takeaway is that DGL matches or exceeds baseline decision quality while delivering substantially faster runtimes without relying on such trade-offs.\looseness=-1

\subsection{Experimental Settings}

\begin{figure*}[t]
    \centering
    \includegraphics[width=\linewidth,trim={0pt 0pt 0 0pt}, clip]{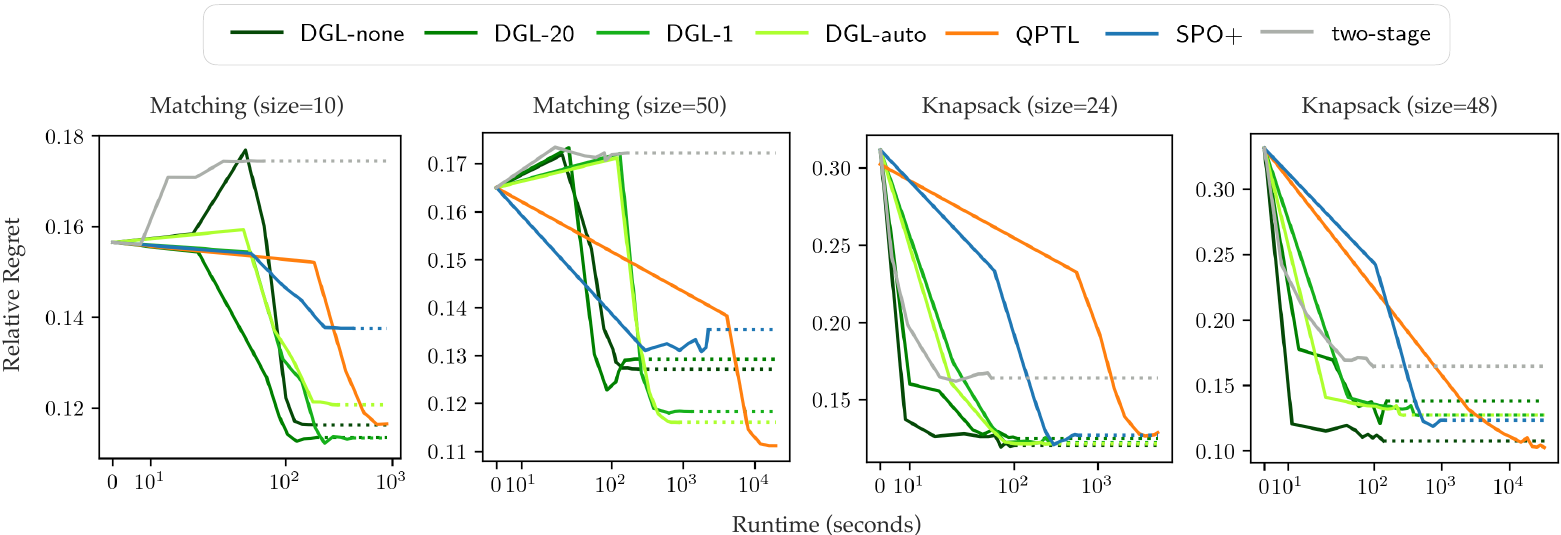}
    \caption{\textbf{Training time vs test relative regret.} Across Matching (sizes 10, 50) and Knapsack (sizes 24, 48), DGL variants reach competitive or better regret far faster than QPTL and SPO+.}
    \label{fig:experiments}
\end{figure*}

\paragraph{Many-to-One Matching.}
We generate synthetic instances inspired by algorithmic refugee assignment \citep{bansak_outcome-driven_2024}. For each individual $i$, we draw $x_i\in\mathbb{R}^{p_{\mathrm{base}}+p_{\mathrm{loc}}}$ from a zero-mean Gaussian with feature-wise standard deviations sampled i.i.d. from $(0,5]$. We take $p_{\mathrm{base}}=10$ and $p_{\mathrm{loc}}=5$ for a total of 15 features. For each location $j$, we form pairwise inputs by repeating $x_i$ across locations and appending one-hot indicators that activate the $p_{\mathrm{loc}}$ location-specific features. Utilities are given by $u_{ij}=\sigma\!\big(y^{\mathrm{base}}_i+\alpha\,y^{\mathrm{loc}}_{ij}+\varepsilon_{ij}\big)$, where $y^{\mathrm{base}}_i$ is a zero-mean, unit-variance \emph{simple} linear combination of base features using integer coefficients in $\{1,\dots,4\}$ (computed once per individual and repeated across locations), and $y^{\mathrm{loc}}_{ij}$ is a \emph{complex} location-differentiated term: for each location $j$, we use integer coefficients in $\{-5,\dots,5\}$ applied to the absolute values of the location features together with a location-specific bilinear interaction $x^{\mathrm{loc}}_{i,0}x^{\mathrm{loc}}_{i,1}$. The resulting matrix is then standardized using dataset-wide mean and standard deviation across all $(i,j)$. Here $\alpha$ weights the location effect, $\varepsilon_{ij}\sim\mathcal{N}(0,\sigma^2)$ is optional Gaussian noise, and $\sigma(\cdot)$ is the sigmoid, yielding $u_{ij}\in(0,1)$. We evaluate two sizes—matching 10 or 50 individuals to 3 locations—with 400/400/200 train/val/test instances per size, using logistic regression implemented as a neural network model with no hidden layers as the predictor.

\paragraph{Weighted Knapsack.} We evaluate on the weighted knapsack benchmark from \cite{mandi_smart_2020}, which uses data from the Irish Single Electricity Market Operator (SEMO) \citep{ifrim_properties_2012}. The tasks consists in selecting half-hour slots that maximize revenue subject to budget constraints. Weights $w_i$ for each time-slot and the total capacity are known, while per-item values $c_i$ must be predicted from features (8 features per slot). Item weights are sampled from $\{ 3,5,7 \}$ to induce correlation between weights and values, each slot’s value is generated by multiplying the energy-price feature vector with the weight and adding Gaussian noise $\varepsilon \sim N(0,25)$. In the released data the total weight sums to 240 and capacities of 60, 120, and 180 are considered. We consider an additional setting with 24 items by taking only the first half-hour of each hour, while retaining the same weight sampling and value-generation scheme. The predictive model is the same used by \cite{mandi_smart_2020} for this task; a simple linear model implemented as a neural network model with no hidden layers.\looseness=-1

\subsection{Results}

Across all four settings (Matching 10-50 and Knapsack 24-48) DGL variants reach low test regret far sooner than solver-in-the-loop baselines (Fig. 2). The DGL curves drop steeply at the start of training, whereas SPO+ and QPTL require much longer time to attain comparable regret. On the larger instances (Matching-50, Knapsack-48), QPTL eventually reaches the lowest regret but only by a modest margin and at a steep computational cost: attaining that final edge demands on the order of two additional orders of magnitude wall-clock time relative to DGL’s early plateau. Consequently, for any reasonable time budget, DGL delivers clearly lower regret than the in-loop methods.\looseness=-1

As instances grow, the feasible set expands combinatorially and per-step solves in SPO+ and QPTL dominate runtime. DGL amortizes solver work and decouples solves from gradient steps, preserving near–two-stage speed even on larger problems. In regimes where an extra $10$--$100\times$ of compute is costly, DGL attains competitive—often better at equal time—decision quality far sooner. When the absolute best final regret is required, one can warm-start QPTL with a short DGL phase to capture the remaining small gains.\looseness=-1

Among DGL variants, light-touch refresh policies
keep the solver mostly out of the loop, so runtime stays near two-stage rather than creeping toward in-loop costs. In our plots, DGL-20 often tracks DGL-none for much of training (and sometimes lies between DGL-1 and DGL-none), while DGL-auto remains competitive by updating only on prospective constraint violations. Even DGL-none (a single initial solve) achieves strong regret quickly—approaching the runtime of two-stage training while remaining decision-aware.\looseness=-1

\section{Discussion}
Our results demonstrate that DGL achieves both high quality decisions and computational efficiency. By using periodically refreshed dual variables as decision-aware signals, DGL eliminates the need for repeated solver calls during most training steps, scaling to problem sizes where existing solver-in-the-loop methods (SPO+, QPTL) become prohibitively expensive. On two benchmark tasks, DGL matches or exceeds state-of-the-art decision quality while substantially reducing runtime, making it a practical choice for real-world PtO pipelines where scalability is paramount.\looseness=-1

Methodologically, our framework highlights the value of dual variables as training signals. In contrast to caching or learned surrogates, which trade off performance for speed, DGL directly uses periodically refreshed dual multipliers in a fixed differentiable loss function. This design preserves close alignment with the downstream task objective while sharply reducing computational overhead. At the same time, DGL has certain limitations. DGL is specifically designed for a particular type of downstream task in which ``pick-one'' groups naturally arise. While many common decision problems can be cast in this form, generalization is still limited. Additionally, even among problems within this class, the empirical benefits of DGL hinge on the informativeness and stability of dual variables, which may vary across applications.\looseness=-1 

In summary, this work contributes a new perspective on decision-focused learning: by turning to the dual side of optimization, we can substantially reduce computational cost while preserving decision alignment. We hope this opens the door to more scalable and practical decision-aware learning pipelines.\looseness=-1 

\bibliographystyle{plainnat}
\bibliography{references.bib}

@inproceedings{shah_decision-focused_2022,
	title = {Decision-{Focused} {Learning} without {Differentiable} {Optimization}: {Learning} {Locally} {Optimized} {Decision} {Losses}},
	shorttitle = {Decision-focused learning without differentiable optimization},
	booktitle = {Proceedings of the 36th {International} {Conference} on {Neural} {Information} {Processing} {Systems}},
	author = {Shah, Sanket and Wang, Kai and Wilder, Bryan and Perrault, Andrew and Tambe, Milind},
	year = {2022},
}

@article{rong_lagrangian_2022,
	title = {A {Lagrangian} dual-based theory-guided deep neural network},
	volume = {8},
	number = {6},
	journal = {Complex \& Intelligent Systems},
	author = {Rong, Miao and Zhang, Dongxiao and Wang, Nanzhe},
	year = {2022},
	pages = {4849--4862},
}

@incollection{nandwani_primal-dual_2019,
	title = {A primal-dual formulation for deep learning with constraints},
	number = {1091},
	booktitle = {Proceedings of the 33rd {International} {Conference} on {Neural} {Information} {Processing} {Systems}},
	publisher = {Curran Associates Inc.},
	author = {Nandwani, Yatin and Pathak, Abhishek and Mausam and Singla, Parag},
	year = {2019},
	pages = {12179--12190},
}

@inproceedings{mulamba_contrastive_2021,
	title = {Contrastive {Losses} and {Solution} {Caching} for {Predict}-and-{Optimize}},
	booktitle = {Proceedings of the {Thirtieth} {International} {Joint} {Conference} on {Artificial} {Intelligence}},
	author = {Mulamba, Maxime and Mandi, Jayanta and Diligenti, Michelangelo and Lombardi, Michele and Bucarey, Victor and Guns, Tias},
	year = {2021},
}

@article{mandi_decision-focused_2024,
	title = {Decision-{Focused} {Learning}: {Foundations}, {State} of the {Art}, {Benchmark} and {Future} {Opportunities}},
	shorttitle = {Decision-{Focused} {Learning}},
	journal = {Journal of Artificial Intelligence Research},
	author = {Mandi, Jayanta and Kotary, James and Berden, Senne and Mulamba, Maxime and Bucarey, Victor and Guns, Tias and Fioretto, Ferdinando},
	year = {2024},
}

@article{mandi_smart_2020,
	title = {Smart {Predict}-and-{Optimize} for {Hard} {Combinatorial} {Optimization} {Problems}},
	urldate = {2025-04-03},
	journal = {Proceedings of the AAAI Conference on Artificial Intelligence},
	author = {Mandi, Jayanta and Demirovi, Emir and Stuckey, Peter J. and Guns, Tias},
	year = {2020},
}

@inproceedings{ifrim_properties_2012,
	title = {Properties of {Energy}-{Price} {Forecasts} for {Scheduling}},
	booktitle = {Principles and {Practice} of {Constraint} {Programming}},
	author = {Ifrim, Georgiana and O’Sullivan, Barry and Simonis, Helmut},
	year = {2012},
	pages = {957--972},
}

@inproceedings{fioretto_lagrangian_2020,
	title = {Lagrangian {Duality} for {Constrained} {Deep} {Learning}},
	booktitle = {Machine {Learning} and {Knowledge} {Discovery} in {Databases}.},
	author = {Fioretto, Ferdinando and Van Hentenryck, Pascal and Mak, Terrence W. K. and Tran, Cuong and Baldo, Federico and Lombardi, Michele},
	year = {2020},
}

@inproceedings{cotter_training_2019,
	title = {Training {Well}-{Generalizing} {Classifiers} for {Fairness} {Metrics} and {Other} {Data}-{Dependent} {Constraints}},
	booktitle = {Proceedings of the 36th {International} {Conference} on {Machine} {Learning}},
	publisher = {PMLR},
	author = {Cotter, Andrew and Gupta, Maya and Jiang, Heinrich and Srebro, Nathan and Sridharan, Karthik and Wang, Serena and Woodworth, Blake and You, Seungil},
	year = {2019},
	pages = {1397--1405},
}

@article{elmachtoub_smart_2022,
	title = {Smart “{Predict}, then {Optimize}”},
	volume = {68},
	number = {1},
	journal = {Management Science},
	author = {Elmachtoub, Adam N. and Grigas, Paul},
	year = {2022},
	pages = {9--26},
}

@article{chang_structured_2012,
	title = {Structured {Learning} with {Constrained} {Conditional} {Models}},
	volume = {88},
	issn = {0885-6125},
	number = {3},
	journal = {Machine Learning},
	author = {Chang, Ming-Wei and Ratinov, Lev and Roth, Dan},
	year = {2012},
	pages = {399--431},
}

@article{bansak_outcome-driven_2024,
	title = {Outcome-{Driven} {Dynamic} {Refugee} {Assignment} with {Allocation} {Balancing}},
	volume = {72},
	issn = {1526-5463},
	number = {6},
	journal = {Operations Research},
	author = {Bansak, Kirk and Paulson, Elisabeth},
	year = {2024},
	pages = {2375--2390},
}

@inproceedings{amos_optnet_2017,
	title = {{OptNet}: {Differentiable} {Optimization} as a {Layer} in {Neural} {Networks}},
	booktitle = {Proceedings of the 34th {International} {Conference} on {Machine} {Learning}},
	author = {Amos, Brandon and Kolter, J. Zico},
	year = {2017},
}

@article{agrawal_dynamic_2014,
	title = {A {Dynamic} {Near}-{Optimal} {Algorithm} for {Online} {Linear} {Programming}},
	volume = {62},
	number = {4},
	journal = {Operations Research},
	author = {Agrawal, Shipra and Wang, Zizhuo and Ye, Yinyu},
	year = {2014},
	pages = {876--890},
}

@inproceedings{agarwal_reductions_2018,
	title = {A {Reductions} {Approach} to {Fair} {Classification}},
	booktitle = {Proceedings of the 35th {International} {Conference} on {Machine} {Learning}},
	publisher = {PMLR},
	author = {Agarwal, Alekh and Beygelzimer, Alina and Dudik, Miroslav and Langford, John and Wallach, Hanna},
	year = {2018},
	pages = {60--69},
}

@article{kelly_rate_1998,
	title = {Rate control for communication networks: shadow prices, proportional fairness and stability},
	volume = {49},
	language = {en},
	journal = {Journal of the Operational Research Society},
	author = {Kelly, FP and Maulloo, AK and Tan, DKH},
	year = {1998},
	pages = {237--252},
}

@article{ganchev_posterior_nodate,
	title = {Posterior {Regularization} for {Structured} {Latent} {Variable} {Models}},
	language = {en},
	author = {Ganchev, Kuzman and Graça, João and Gillenwater, Jennifer and Taskar, Ben},
}

@inproceedings{wilder_melding_2019,
	title = {Melding the data-decisions pipeline: decision-focused learning for combinatorial optimization},
	booktitle = {Proceedings of the {Thirty}-{Third} {AAAI} {Conference} on {Artificial} {Intelligence}},
	author = {Wilder, Bryan and Dilkina, Bistra and Tambe, Milind},
	year = {2019},
}

@inproceedings{wang_scalable_2023,
	title = {Scalable decision-focused learning in restless multi-armed bandits with application to maternal and child health},
	booktitle = {Proceedings of the {Thirty}-{Seventh} {AAAI} {Conference} on {Artificial} {Intelligence}},
	author = {Wang, Kai and Verma, Shresth and Mate, Aditya and Shah, Sanket and Taneja, Aparna and Madhiwalla, Neha and Hegde, Aparna and Tambe, Milind},
	year = {2023},
}

\clearpage
\appendix
\thispagestyle{empty}
\setcounter{figure}{0}
\renewcommand{\thefigure}{A\arabic{figure}}

\onecolumn
\aistatstitle{A Dual Perspective on Decision-Focused Learning:\\ Scalable Training via Dual-Guided Surrogates\\
(Supplementary Materials)}
\section{``Pick-one" group problem formulations}
\label{app:formulations}
\subsection{Assignment Problem}

Consider the problem of assigning $m$ workers to $k$ jobs, where each job has capacity $c_k$. The match quality $w_{lj}$ between a worker and job must be estimated.

The assignment problem is typically written as follows:
$$
\max_{\bz\in\{0,1\}^{m\times k}}
      \sum_{l,j}w_{lj}z_{lj}
\quad
\text{s.t.}
\;
\sum_{j}z_{lj}=1,\;
\sum_{l}z_{lj}\le c_j.
$$

To map this formulation to the Problem \ref{eq:original_LP}, consider vectorizing the decision variable $\mathbf{\bz}$:
$$
\bx \triangleq (z_{11},\dots,z_{1k},\,z_{21},\dots,z_{mk})^\top\in\{0,1\}^{n},
\quad n=mk,
$$
and similarly let $\mathbf{\by}$ be the vectorized version of $\mathbf{w}$.

Each \emph{worker row} is a group, since each worker can only be assigned to one job:
$$g_l=\bigl\{(l-1)k+1,...,lk\bigr\},\qquad l=1,\dots,m.$$

The global constraint $\bA\bx\le \bb$ must be constructed to reflect the capacity constraints 
$\sum_{l}z_{lj}\le c_j$. Therefore, $\bA\in\mathbb{R}^{k\times mk}$ has $A_{j,(l-1)k+j}=1$ for $j=1,...,k$ and $l=1,...,m$ and is otherwise equal to zero. Additionally, $\bb=(c_1,...,c_k)^\top$.

Now, notice that $(\bA^{\top}\bl)_{j\in g_l}=(\lambda_1,...,\lambda_k)$ and therefore $h_i(\by,\bl)=y_i-\lambda_i$.

\subsection{Weighted Knapsack Problem}

The classic weighted knapsack problem consists of items $i\in \{1,\dots,m\}$ with value $v_i$ and weight $w_i$. 
To formulate this as above, we must be explicit about the item groups. Each item has two options: add to the knapsack or do not. Thus, each item will be in a group of size two.

The knapsack problem can be written as:
$$
\max_{\bx\in\{0,1\}^{m\times 2}}\;\sum_{i}v_{i} x_{i1}
\quad
\text{s.t.}\;
\sum_{i}w_i x_{i1} \le B \text{ for } j=1,2, \quad \sum_{j=1,2} x_{ij}=1 \text{ for all } i,
$$
where $x_{i1}=1$ if item $i$ is added to the knapsack and $x_{i2}=1$ if it is not.
Additionally, $B$ is the knapsack's capacity. As with the Assignment Problem, to map this problem to the general approach described above we will flatten the decision variables into a vector $\bz$, with $\by=(v_1,0,v_2,0,...,v_m,0)$ and $\bA=(w_1,0,w_2,0,...,w_m,0)$. We also have that $g_l=\{2(l-1)+1, 2l\}$.

We can then show that $(h_i(\mathbf{\by},\bl))_{i\in g_l}=(y_{2(l-1)+1} - w_{2(l-1)+1}\bl_1, 0)$. This is because for any group (i.e., item) $l$, the second option of \emph{not} adding the item to the knapsack has a value of zero and a dual variable of zero.
In other words, we should add item $i$ to the knapsack if $y_i-w_i\bl_1\geq 0$.

\section{Proof of Theorem 1}

We begin by proving a lemma that bounds the gap between softmax and argmax.
\begin{lemma}
\label{lem:soft-hard-gap}
Fix any vector $\mathbf{u}\in\mathbb{R}^g$. Let $j^\star=\arg\max_{i} u_i$ and define the \emph{margin}
\[
\Delta(\mathbf{u})
:= u_{j^\star}-\max_{i\neq j^\star} u_i \;>\;0.
\]
Let $p_{\tau,i}(\mathbf{u})$ be the softmax probability assigned to 
$u_i$: 
\[
p_{\tau,i}(\mathbf{u})
=\frac{e^{u_i/\tau}}{\sum_{j} e^{u_j/\tau}}.
\]


Then, the $\ell_1$ distance between the one--hot vector $e_{j^\star}$ and
the softmax vector on $g$ satisfies
\begin{equation}
\bigl\|\,\mathbf{e}_{j^\star}-\mathbf{p}_{\tau}\,\bigr\|_1
\;\le\; 2(|g|-1)\,e^{-\Delta(\mathbf{u})/\tau}.
\label{eq:l1-gap}
\end{equation}
\end{lemma}

\begin{proof}

Notice that 
$p_{\tau, j^\star}$ can be written as:
\[
p_{\tau,j^*}(\mathbf{u})
=\frac{e^{u_{j^*}/\tau}}{\sum_{i} e^{u_i/\tau}}
=\frac{1}{1+\sum_{i\neq j^\star} e^{-(u_{j^\star}-u_i)/\tau}}.
\]

Set
\[
S(\mathbf{u},\tau)
:=\sum_{i\neq j^*}
              \exp\Big(\frac{u_i-u_{j^\star}}{\tau}\Big)
=\sum_{i\ne j^\star}\exp\!\Big(-\frac{u_{j^\star}-u_i}{\tau}\Big)
\;\ge 0.
\]

Then $p_{\tau,j^*}(\mathbf{u})=\frac{1}{1+S(u,\tau)}$. Using the
inequality $\frac{1}{1+r}\ge 1-r$ for all $r\ge 0$,
\[
p_{\tau,j^*}(\mathbf{u})\;\ge\;1-S(\mathbf{u},\tau)
\quad\Longrightarrow\quad
1-p_{\tau,j^*}(\mathbf{u})\;\le\;S(\mathbf{u},\tau).
\]
Moreover, since $u_{j^\star}-u_i\ge \Delta(\mathbf{u})$ for all
$i\neq j^\star$, each summand obeys
\[
\exp\!\Big(-\frac{u_{j^\star}-u_i}{\tau}\Big)
\;\le\;
\exp\!\Big(-\frac{\Delta(\mathbf{u})}{\tau}\Big),
\]
and there are exactly $|g|-1$ such terms. Hence
\[
S(\mathbf{u},\tau)\;\le\;(|g|-1)\,e^{-\Delta(\mathbf{u})/\tau}.
\]
Thus,
\[
1-p_{\tau,j^*}(\mathbf{u})
\;\le\; \sum_{i\neq j^\star} e^{-(u_{j^\star}-u_i)/\tau}
\;\le\; (|g|-1)\,e^{-\Delta(\mathbf{u})/\tau}.
\]

Using the fact that $\sum_{i}p_{\tau,i}(\mathbf{u})=1$,
\[
\bigl\|\mathbf{e}_{j^\star}-\tilde \bz_{\tau}(\mathbf{u})\bigr\|_1
= |1-p_{\tau,j^*}(\mathbf{u})|+\sum_{i\ne j^\star}|p_{\tau,i}(\mathbf{u})|
= (1-p_{\tau,j^*}(\mathbf{u})) + \sum_{i\ne j^\star} p_{\tau,i}(\mathbf{u})
= 2\,(1-p_{\tau,j^*}(\mathbf{u})).
\]
Applying the bound above gives \eqref{eq:l1-gap}:
\[
\bigl\|\mathbf{e}_{j^\star}-\mathbf{p}_{\tau}(\mathbf{u})\bigr\|_1
= 2\,(1-p_{\tau,j^*}(\mathbf{u}))
\;\le\; 2(|g|-1)\,e^{-\Delta(\mathbf{u})/\tau}.
\]
This completes the proof. \hfill $\square$
\end{proof}

\begin{lemma}\label{lemma2}
Let $c \in \{\by, \hat \by\}$. Under Assumptions~A1 and~A2, the LP relaxation of Problem~1 admits an optimal solution $\bx^{\mathrm{LP}}(c)$ satisfying
\[
\bx^{\mathrm{LP}}(c) \;=\; \bz^*(c),
\]
where $\bz^*(c)$ is the groupwise argmax vector defined by
\[
\bz^*_{g}(c) \;=\; e_{\arg\max_{i \in g} h_i(c)}, \qquad 
\bh(c) \;=\; c - \bA^\top \bl^*(c),
\]
for any dual--optimal $\bl^*(c)$. Moreover, by Assumption~A1,
\[
\bx^{\mathrm{IP}}(c) \;=\; \bx^{\mathrm{LP}}(c) \;=\; \bz^*(c).
\]
\end{lemma}

\begin{proof}
Fix $c \in \{\by, \hat \by\}$ and consider the LP relaxation of Problem~1:
\[
\begin{aligned}
\max_{\bx \in \R^n}\quad & c^\top \bx \\
\text{s.t.}\quad & \bA \bx \le \bb, \\
& \sum_{i \in g} x_i = 1 \quad \forall g \in \mathcal{G}, \\
& \bx \ge 0.
\end{aligned}
\]
The corresponding dual has variables $\bl \ge 0$ for the global constraints
and $\{\mu_g \in \R : g \in \mathcal{G}\}$ for the group constraints:
\[
\begin{aligned}
\min_{\bl,\mu}\quad & \bb^\top \bl + \sum_{g \in \mathcal{G}} \mu_g \\
\text{s.t.}\quad & \bA^\top \bl + \tilde\mu \;\ge\; c.
\end{aligned}
\]
where $\tilde\mu \in \R^n$ is defined by $\tilde\mu_i = \mu_{g(i)}$ for the group $g(i)$
containing $i$.

By strong duality there exists an optimal primal--dual pair
$(\bx^{\mathrm{LP}}(c),\bl^*(c),\mu^*(c))$ satisfying the KKT conditions.
Define the reduced costs
\[
\bh(c) := c - \bA^\top \bl^*(c).
\]

Dual feasibility gives
\[
h_i(c) \;\le\; \mu^*_{g(i)} \quad \forall i.
\]
Complementary slackness implies that if $\bx^{\mathrm{LP}}_i(c) > 0$ then
\[
h_i(c) \;=\; \mu^*_{g(i)}.
\]
Thus, within each group $g$, the support of $\bx^{\mathrm{LP}}(c)$ is contained in
$\arg\max_{i \in g} h_i(c)$.

By (A2), for each $g$ the maximizer of $h_i(c)$ is unique almost surely,
say $j_g^*$. Since $\sum_{i \in g} x^{\mathrm{LP}}_i(c)=1$ and $x^{\mathrm{LP}}_i(c)\ge 0$,
the only feasible choice is
\[
x^{\mathrm{LP}}_i(c) =
\begin{cases}
1 & i = j_g^*, \\
0 & i \in g \setminus \{j_g^*\}.
\end{cases}
\]
Therefore $\bx^{\mathrm{LP}}(c) = \bz^*(c)$.

(A1) states that the LP relaxation of Problem~1 is integral,
so any LP--optimal solution coincides with an integer feasible solution
of the IP. Hence
\[
\bx^{\mathrm{IP}}(c) = \bx^{\mathrm{LP}}(c).
\]

Combining the steps proves that
\[
\bx^{\mathrm{IP}}(c) = \bx^{\mathrm{LP}}(c) = \bz^*(c).
\] \hfill $\square$
\end{proof}

\textbf{Proof of Theorem \ref{thm:regret}.}
Let $\hat \bx:=\bx^{\ip}(\hat \by)$ and $\bx^{\star}:=\bx^{\ip}(\by)$.
Add and subtract the softmax decision values at $\hat \by$ and $\by$:
\begin{align*}
\mathrm{Regret}(\theta)
&= \mathbb E\bigl[ \by^{\top}\bx^{\star}- \by^{\top}\hat \bx  \bigr] \\
&= \mathbb E\bigl[ \by^{\top} \bx^{\star} - \by^{\top}\tilde \bz_{\tau}(\by) \bigr]
   - \mathbb E\bigl[ -\by^{\top}\tilde \bz_{\tau}(\by) \bigr] \\
&\quad+ \mathbb E\bigl[ \by^{\top}\tilde \bz_{\tau}(\hat \by) - \by^{\top}\hat \bx \bigr]
   + \mathbb E\!\bigl[ -\by^{\top}\tilde \bz_{\tau}(\hat \by) \bigr].
\end{align*}

Recognizing $\mathbb E[-\by^{\top}\tilde \bz_{\tau}(\hat \by)]=|\mathcal G|\,\mathcal L_{\tau}(\theta)$ and
$\mathcal L_{\tau}^{*}\le -\frac{1}{|\mathcal G|}\mathbb E[\by^{\top}\tilde \bz_{\tau}(\by)]$
(by plugging in $\by$ for $\hat \by$), we get
$$
\mathrm{Regret}(\theta)
\leq 
|\mathcal G|\bigl[\mathcal L_{\tau}(\theta)-\mathcal L_{\tau}^{*}\bigr]
+
\mathbb E\bigl[\by^{\top}\bx^{\star} - \by^{\top}\tilde \bz_{\tau}(\by)\bigr]
+
\mathbb E\bigl[\by^{\top}\tilde \bz_{\tau}(\hat{\by}) - \by^{\top}\hat{\bx}\bigr].
$$

By (A1) and (A2), for $c\in\{\by, \hat \by\}$ the optimal value of the LP satisfies $\bx^{\text{LP}}(c)=\bz^*(c)$ by the KKT conditions and uniqueness, and (A1) implies that $\bx^{\text{LP}}(c)=x^{\ip}(c)$. This is formally shown in Lemma \ref{lemma2}. Hence, we can apply Lemma \ref{lem:soft-hard-gap} group-wise to $\bx^{\ip}(c)$ and $\tilde{\bz}_{\tau}(c)$ to obtain 

$$
\mathbb E\bigl[\bigl|\by^{\top}\hat \bx-\by^{\top}\tilde \bz_{\tau}(\hat \by)\bigr|\bigr]
\leq  \delta_{\hat \by}(\tau),
\qquad
\mathbb E\!\bigl[\bigl|\by^{\top}\bx^{\star}-\by^{\top}\tilde \bz_{\tau}(\by)\bigr|\bigr]
\leq \delta_{\by}(\tau),
$$
where
$$
\delta_{\hat \by}(\tau)
=2C\,\mathbb E\!\Big[\sum_{g\in\mathcal G} (|g|-1)\,e^{-\Delta_g(\bh(\hat \by))/\tau}\Big],
\qquad
\delta_{\by}(\tau)
=2C\,\mathbb E\!\Big[\sum_{g\in\mathcal G} (|g|-1)\,e^{-\Delta_g(\bh(\by))/\tau}\Big].
$$
Assumption~A2 ensures $\Delta_g(\cdot)>0$ almost surely, so
$\delta_{\hat \by}(\tau)\to0$ and $\delta_{\by}(\tau)\to0$ as $\tau\downarrow0$.

Therefore, we have that 
$$
\mathrm{Regret}(\theta)
\leq 
|\mathcal G|\bigl[\mathcal L_{\tau}(\theta)-\mathcal L_{\tau}^{*}\bigr]
+
2C\,\mathbb E\!\Big[\sum_{g\in\mathcal G} (|g|-1)\,e^{-\Delta_g(\hat \by)/\tau}\Big] + 2C\,\mathbb E\!\Big[\sum_{g\in\mathcal G} (|g|-1)\,e^{-\Delta_g(\by)/\tau}\Big]
$$

Applying A2, we have that
$$\mathrm{Regret}(\theta)
\le |\mathcal G|[\mathcal L_\tau(\theta)-\mathcal L_\tau^{*}] + O(e^{-\gamma/\tau}).$$
\hfill $\square$

\section{Regret Bound for Modified Loss Function $\mathcal{J}_{\tau}^{(\alpha)}(\theta)$}
In this section we bound the decision regret in terms of the loss function $\mathcal{J}_{\tau}^{(\alpha)}(\theta)$. 
Throughout, we work at a small but positive temperature
$0<\tau\le\tau_0$ and use the notation $O(\tau)$ to denote a constant
(which may depend on $C$ and $|\mathcal G|$) times~$\tau$.

Recall the definition of $\tilde{\ell}_{\tau}(\theta)$, given by
\begin{equation*}
    \tilde{\ell}_{\tau}(\hat \by(\theta))\triangleq -\frac{1}{|\mathcal{G}|} \sum_{g\in\mathcal{G}}(\mathbf{\by}-\bA^{\top}\hat \bl)^{\top} \tilde \bz_{g,\tau}(\hat{\mathbf{\by}}, \hat \bl),
\end{equation*}

and let $\tilde{\mathcal{L}}_{\tau}(\theta)=\mathbb{E}[\tilde{\ell}_{\tau}(\hat \by(\theta))],$ and $\tilde{\mathcal L}_{\tau}^{\star}\;\triangleq\;\inf_{\theta}\,\tilde{\mathcal L}_{\tau}(\theta)$.

Our expected modified loss function is thus given by 
$$\mathcal J_{\tau}^{(\alpha)}(\theta)\triangleq \tilde{\mathcal L}_{\tau}(\theta)\;+\;\alpha\,\mathbb{E}\!\left[\frac{1}{n}\sum_{i=1}^n (y_i-\hat{y}_i)^2\right],
$$
and additionally define \(\mathcal J_{\tau}^{(\alpha)\,*}\;\triangleq\;\inf_{\theta}\,\mathcal J_{\tau}^{(\alpha)}(\theta)\).

We are interested in showing that decision regret vanishes when $\mathcal J_{\tau}^{(\alpha)}(\theta)$ approaches $J_{\tau}^{(\alpha)^*}$.

\begin{theorem}
\label{thm:regret-alt}
Assume \textbf{A1--A3}.  Then for every $\tau>0$, $\alpha>0$, and
$\theta$,
\[
\mathrm{Regret}(\theta)
\;\le\;
\sqrt{2|\mathcal G|}\,
\sqrt{\,\mathbb E\!\bigl[\|\by-\hat\by\|_{2}^{2}\bigr]}
\;+\;
|\mathcal G|\,\bigl[\tilde{\mathcal L}_{\tau}(\theta)-\tilde{\mathcal L}^{\star}_{\tau}\bigr]
\;+\;
O(\tau).
\]
Consequently, if a parameter sequence $(\theta_k,\tau_k)$ satisfies  
\[
\mathcal J_{\tau_k}^{(\alpha)}(\theta_k)
\;-\;\mathcal J_{\tau_k}^{(\alpha)\,*}\;\longrightarrow 0,
\qquad
\tau_k\downarrow 0,
\]
then $\mathrm{Regret}(\theta_k)\longrightarrow 0$.
\end{theorem}

\begin{proof}
For brevity let $\bx^{\ip}\triangleq \bx^{\ip}(\by)$ and $\hat \bx^{\ip}\triangleq \bx^{\ip}(\hat \by)$.

\textbf{Step 1. Regret decomposition.} First, decompose regret as:
\begin{align}
\mathbf{\by}^{\top}\bx^{\ip}-\mathbf{\by}^{\top}\hat \bx^{\ip}
&=
(\by-\hat \by)^{\top}\bx^{\ip}
  - (\by-\hat \by)^{\top}\hat \bx^{\ip}
  + \hat \by^{\top}\bx^{\ip}-\hat \by^{\top} \hat \bx^{\ip}\\[2pt]
&=
\underbrace{(\by-\hat \by)^{\top}(\bx^{\ip}-\hat \bx^{\ip})}_{\text{Term 1}}
\;+\;
\underbrace{\hat \by^{\top}(\bx^{\ip}-\hat \bx^{\ip})}_{\text{Term 2}}. \label{eq:decomp}
\end{align}

\textbf{Step 2. Bound Term 1.} Because each group contributes at most two nonzeros in 
$\bx^{\ip}-\hat \bx^{\ip}$, we have
$\|\bx^{\ip}-\hat \bx^{\ip}\|_2 \le \sqrt{2|\mathcal G|}$.  
By Cauchy--Schwarz,
\[
\bigl|(\by-\hat \by)^{\top}(\bx^{\ip}-\hat \bx^{\ip})\bigr|
\;\le\; \sqrt{2|\mathcal G|}\,\|\by-\hat \by\|_2.
\]

\textbf{Step 3. Bound Term 2.} Let $\hat\bl$ be dual-optimal for $\hat \by$ and set
$\hat \bh=\hat \by-\bA^{\top}\hat\bl$.
We can write
\begin{align}
\hat \by^{\top}(\bx^{\ip}-\hat \bx^{\ip})
&=\bigl(\hat \bh+\bA^{\top}\hat\bl\bigr)^{\top}(\bx^{\ip}-\hat \bx^{\ip})\\[2pt]
&=\hat \bh^{\top}(\bx^{\ip}-\hat \bx^{\ip})
  +\hat\bl^{\top}(\bA\bx^{\ip}-\bA\hat \bx^{\ip}) \\
&=\hat \bh^{\top}(\bx^{\ip}-\hat \bx^{\ip})
  +\hat\bl^{\top}(\bA\bx^{\ip}-\bb)  
\end{align}

where the final equality follows because $(\hat \bx^{\ip},\hat\bl)$ is primal–dual optimal (by A1 and strong duality),
complementary slackness gives
$\hat\bl^{\top}\bigl(\bb-\bA\hat \bx^{\ip}\bigr)=0$
and thus
$\hat\bl^{\top}\bA\hat \bx^{\ip}=\hat\bl^{\top}\bb$.  

Since $\bA\bx^{\ip}\le \bb$ and $\hat\bl\ge 0$,
the last term is non-positive and therefore

\begin{align}\label{eq:ub}
& \hat \by^{\top}(\bx^{\ip}-\hat \bx^{\ip})
\;\le\;
\hat \hat \bh^{\top}(\bx^{\ip}-\hat \bx^{\ip})
\end{align}

Recall that $\tilde{\bz}_{\tau}$ is the group-wise softmax transformation of $\hat \bh$.
Add and subtract $\hat \bh^{\top}\tilde \bz_\tau$ from the upper bound in \ref{eq:ub}:
\[
\hat \bh^{\top}(\bx^{\ip}-\hat \bx^{\ip})
= \hat \bh^{\top}(\bx^{\ip}-\tilde \bz_\tau)
 +\hat \bh^{\top}(\tilde \bz_\tau-\hat \bx^{\ip}).
\]

From the first term above, add and subtract $\bh^{\top}\bz^*$ to obtain
$$ \hat \bh^{\top}(\bx^{\ip}-\hat \bx^{\ip})
= \hat \bh^{\top}(\bx^{\ip}-\bz^*) + \hat \bh^{\top}(\bz^*-\tilde \bz_\tau)
 +\hat \bh^{\top}(\tilde \bz_\tau-\hat \bx^{\ip}).$$

Because $\hat \bh^{\top}\bz^*=\argmax\hat \bh$, $\hat \bh^{\top}(\bx^{\ip}-\bz^*)\leq 0$ and therefore

$$ \hat \bh^{\top}(\bx^{\ip}-\hat \bx^{\ip})
\leq \underbrace{\hat \bh^{\top}(\bz^*-\tilde \bz_\tau)}_{\text{Term 2(a)}}
 +\underbrace{\hat \bh^{\top}(\tilde \bz_\tau-\hat \bx^{\ip})}_{\text{Term 2(b)}}.$$

\textbf{Step 3a. Bound Term 2(a).}  Define $\bh^{\dagger}:=\by-\bA^{\top}\hat\bl$. This is the term that is used as a proxy for reward in $\tilde{\ell}_{\tau}(\theta)$. Using $\hat \bh-\bh^{\dagger}=\hat \by-\by$, we have for any vectors $\mathbf{u},\mathbf{v}$,
\[
\hat \bh^{\top}(\mathbf{u}-\mathbf{v})=\bh^{{\dagger}^\top}(\mathbf{u}-\mathbf{v}) + (\hat \by-\by)^{\top}(\mathbf{u}-\mathbf{v}).
\]

Applying the identity above with $\mathbf{u}=\bz^*$ (hard argmax under $\hat \bh$) and $\mathbf{v}=\tilde \bz_\tau$:
\begin{align*}
\hat \bh^{\top}(\bz^*-\tilde \bz_\tau)
&=\bh^{{\dagger}^\top}(\bz^*-\tilde \bz_\tau) + (\hat \by-\by)^{\top}(\bz^*-\tilde \bz_\tau) \\
&\le \bh^{{\dagger}^\top}(\bz^\dagger-\tilde \bz_\tau) + (\hat \by-\by)^{\top}(\bz^*-\tilde \bz_\tau),
\end{align*}
where $\bz^\dagger$ is the hard argmax under $\bh^{\dagger}$,
and the inequality uses that $\bh^{{\dagger}^\top} \bz^*\le \bh^{{\dagger}^\top} \bz^\dagger$ groupwise (which follows from the fact that $\bh^{{\dagger}^\top} \bz^\dagger = \argmax \bh^{{\dagger}^\top}$).

By definition of the expected loss, taking expectations gives
\[
\mathbb{E}\bigl[\bh^{{\dagger}\top}(\bz^\dagger-\tilde \bz_\tau)\bigr]
\;=\;|\mathcal G|\bigl[\tilde{\mathcal L}_\tau(\theta)-\tilde{\mathcal L}^{\star}_{\tau}\bigr].
\]

The correction $(\hat \by-\by)^{\top}(\bz^*-\tilde \bz_\tau)$ is bounded by
Cauchy--Schwarz:
\[
\bigl|(\hat \by-\by)^{\top}(\bz^*-\tilde \bz_\tau)\bigr|
\;\le\;\|\hat \by-\by\|_2 \,\|\bz^*-\tilde \bz_\tau\|_2.
\]
Because each $\tilde \bz_{g,\tau}$ puts $1-O(\tau)$ on its maximiser and 
$O(\tau)$ elsewhere, 
$\|\bz^*-\tilde \bz_\tau\|_2=O(\tau)\sqrt{|\mathcal G|}$.
Thus in expectation,
\[
\mathbb E\bigl|(\hat \by-\by)^{\top}(\bz^*-\tilde \bz_\tau)\bigr|
\;\le\;O(\tau)\sqrt{|\mathcal G|}\,
       \sqrt{\mathbb E\|\by-\hat \by\|_2^2}.
\]

Therefore, the expectation of Term 2(a) is bounded above by 
$$
|\mathcal G|\bigl[\tilde{\mathcal L}_\tau(\theta)-\tilde{\mathcal L}^*(\theta)\bigr]+O(\tau)\sqrt{|\mathcal G|}\,
       \sqrt{\mathbb E\|\by-\hat \by\|_2^2}.$$

\textbf{Step 3b. Bound Term 2(b).} Finally, we must bound the term $\hat \bh^{\top}(\tilde \bz_\tau-\hat \bx^{\ip})$. By Lemma \ref{lemma2}, we have that $\hat \bx^{\ip}=\bz^*$. Therefore, $\tilde \bz_\tau$ becomes close to $\hat \bx^{\ip}$ as $\tau\rightarrow 0$. As before, we have that $|\hat \bh^{\top}(\tilde \bz_\tau-\hat \bx^{\ip})|=O(\tau) |\mathcal{G}|$ since $\hat \bh^{\top}$ is bounded.

\textbf{Step 4. Combining the bounds.} Substituting into the decomposition, then taking expectations,
\begin{align*}
\mathrm{Regret}(\theta)
&\le \sqrt{2|\mathcal G|}\,\sqrt{\mathbb E\|\by-\hat \by\|_2^2} \\
&\quad +|\mathcal G|\bigl[\tilde{\mathcal L}_\tau(\theta)-\tilde{\mathcal L}^{\star}_{\tau}\bigr]
      +O(\tau)\sqrt{|\mathcal G|}\,\sqrt{\mathbb E\|\by-\hat \by\|_2^2}+O(\tau)\,|\mathcal G|.
\end{align*}
Absorb constants into the big-$O$ notation to obtain
\[
\mathrm{Regret}(\theta)
\;\le\;
\sqrt{2|\mathcal G|}\,\sqrt{\mathbb E\|\by-\hat \by\|_2^2}
+|\mathcal G|\bigl[\tilde{\mathcal L}_\tau(\theta)-\tilde{\mathcal L}^{\star}_{\tau}\bigr]
+O(\tau).
\]
\end{proof}

\section{Pseudocode for DGL Dual-Update Variants}
In Section \ref{sec:dual_refresh}, we introduced three DGL variants for handling dual variables during training and presented Algorithm \ref{alg:dgl-fixed} for fixed-frequency updates. Algorithms \ref{alg:dgl-none} and \ref{alg:dgl-auto}, shown here, provide pseudocode for DGL with no dual updates and with rule-based automatic updates, respectively. Relative to Algorithm \ref{alg:dgl-fixed}, the training loop is unchanged; only the dual initialization and lines 3--8 differ. We highlight the variant-specific lines in \textcolor{Green}{green}.

\begin{algorithm}[h!]
\small
\caption{DGL with No Dual Updates}
\label{alg:dgl-none}
\begin{algorithmic}[1]
\REQUIRE Dataset $\{(\mathbf{w}^{(k)}, \mathbf{y}^{(k)})\}_{k=1}^K$, model $M_\theta$, temperature $\tau$, learning rate $\eta$, number of epochs $T$
\STATE Initialize model parameters $\theta$
\STATE \textcolor{Green}{Solve Problem~\ref{eq:original_LP} with true train costs to obtain updated duals
$\bl_k\leftarrow \bl^*(\mathbf{y}^{(k)})$}
\FOR{$t = 0$ to $T-1$}
    \FOR{each mini-batch $\mathcal{B} \subset \{1, \dots, K\}$}
        \STATE Compute predictions 
        $\hat{\mathbf{y}}^{(k)} \leftarrow (M_\theta(\mathbf{w}_{i}^{(k)}))_{i\in \mathcal{B}}$
        \STATE Compute surrogate decisions $\tilde{\mathbf{z}}_k = \text{softmax}(\tau \cdot (\hat{\mathbf{y}}^{(k)} - \bA^\top \bl_k))$ for each $k\in \mathcal{B}$
        \STATE Compute loss $\ell_\tau(\theta, \bl) = \frac{1}{|\mathcal{B}|} \sum_{k \in \mathcal{B}} -\mathbf{y}^{(k)\top} \tilde{\mathbf{z}}_k$
        \STATE Update model: $\theta \leftarrow \theta - \eta \cdot \nabla_\theta \ell_\tau(\theta,\bl)$
    \ENDFOR
\ENDFOR
\end{algorithmic}
\end{algorithm}

\begin{algorithm}[h!]
\small
\caption{DGL with Auto Updates}
\label{alg:dgl-auto}
\begin{algorithmic}[1]
\REQUIRE Dataset $\{(\mathbf{w}^{(k)}, \mathbf{y}^{(k)})\}_{k=1}^K$, model $M_\theta$, temperature $\tau$, learning rate $\eta$, number of epochs $T$, tolerance $\delta$
\STATE Initialize model parameters $\theta$
\textcolor{Green}{ \STATE Initialize duals. Solve Problem~\ref{eq:original_LP} with current predictions
$\bl_k\leftarrow \bl^*(\hat{\mathbf{y}}^{(k)})$}
\FOR{$t = 0$ to $T-1$}
    \FOR{each mini-batch $\mathcal{B} \subset \{1, \dots, K\}$}
        \STATE Compute predictions 
        $\hat{\mathbf{y}}^{(k)} \leftarrow (M_\theta(\mathbf{w}_{i}^{(k)}))_{i\in \mathcal{B}}$
        \STATE Compute surrogate decisions $\tilde{\mathbf{z}}_k = \text{softmax}(\tau \cdot (\hat{\mathbf{y}}^{(k)} - \bA^\top \bl_k))$ for each $k\in \mathcal{B}$
\textcolor{Green}{        \IF{$A^\top \tilde{\mathbf{z}}_k > \bb$ or $A^\top \tilde{\mathbf{z}}_k < \bb-\delta$} 
            \STATE Solve Problem~\ref{eq:original_LP} with current predictions to obtain updated duals
            $\bl_k\leftarrow \bl^*(\hat{\mathbf{y}}^{(k)})$
        \ENDIF}
        \STATE Compute loss $\ell_\tau(\theta, \bl) = \frac{1}{|\mathcal{B}|} \sum_{k \in \mathcal{B}} -\mathbf{y}^{(k)\top} \tilde{\mathbf{z}}_k$
        \STATE Update model: $\theta \leftarrow \theta - \eta \cdot \nabla_\theta \ell_\tau(\theta,\bl)$
    \ENDFOR
\ENDFOR
\end{algorithmic}
\end{algorithm}

\section{Implementation Details}
The code implementing DGL and reproducing the experiments in Section~\ref{sec:experiments} is available at our (anonymized) repository: \href{hidden link}{\color{Blue}{GitHub}}. Our implementation builds on the codebase of \cite{shah_decision-focused_2022} (\href{https://github.com/sanketkshah/LODLs}{\color{Blue}{GitHub}}). Following that architecture, we implemented the task environments \texttt{SyntheticMatching} and \texttt{WeightedKnapsack} as subclasses of \texttt{PThenO}, and added the DGL loss to \texttt{losses.py}.

\paragraph{Hyperparameter selection.}
For each method we perform a grid search over the ranges in Table~\ref{tab:hparam_ranges}. The selected configuration is the one with the lowest \emph{validation} relative regret. Test metrics are reported only for this chosen configuration to avoid validation--test leakage. Table~\ref{tab:selected_hparams} summarizes the selected hyperparameters $(\eta,\alpha,\tau)$ for every experimental setting and method (``—'' denotes not applicable). 

\begin{table}[h!]
\centering
\small
\caption{Hyperparameter search spaces.}
\label{tab:hparam_ranges}
\begin{tabular}{@{} l l l @{}}
\toprule
\textbf{Hyperparameter} & \textbf{Symbol} & \textbf{Range / Values Tried} \\
\midrule
Learning rate & $\eta$ & $\{0.01, 0.05, 0.1, 0.5, 1\}$ \\
Loss weight on MSE & $\alpha$ & $\{0.1, 0.2, \ldots, 1\}$ \\
Temperature & $\tau$ & \{0.1, 0.5, 1 \}\\
\bottomrule
\end{tabular}
\end{table}

\begin{table}[h!]
\centering
\small
\caption{Selected hyperparameters per method and experimental setting.}
\label{tab:selected_hparams}
\setlength{\tabcolsep}{5pt}
\begin{tabular}{@{}l *{12}{c}@{}}
\toprule
& \multicolumn{3}{c}{Matching (size=10)}
& \multicolumn{3}{c}{Matching (size=50)}
& \multicolumn{3}{c}{Knapsack (size=24)}
& \multicolumn{3}{c}{Knapsack (size=48)} \\
\cmidrule(lr){2-4} \cmidrule(lr){5-7} \cmidrule(lr){8-10} \cmidrule(lr){11-13}
\textbf{Method}
& $\eta$ & $\alpha$ & $\tau$
& $\eta$ & $\alpha$ & $\tau$
& $\eta$ & $\alpha$ & $\tau$
& $\eta$ & $\alpha$ & $\tau$ \\
\midrule
Two-stage &$10^{-2}$  & —     & —     &$10^{-2}$  & — & —  &0.1  & —      & —   &1.0  & — & — \\
SPO+      &$10^{-2}$  & —   & —       &$10^{-2}$  & —  & —  &0.1  & —     & —   &0.1  & —    & — \\
QPTL      &$10^{-2}$  & 0.0   & —     &$10^{-2}$  &0.2 & —  &0.5  &0.0    & —   &1.0  &0.0    & — \\
DGL-none  &$10^{-2}$  & 0.0   & 0.1   &$10^{-2}$  &0.6 &0.1  &1.0  &0.0    &0.1 &1.0  &0.0    &1.0    \\
DGL-auto  &$10^{-2}$  & 0.0   & 0.1   &$10^{-2}$  &0.6 &1.0  &1.0  &0.0    &0.1 &1.0  &0.6    &0.1    \\
DGL-1     &$10^{-2}$  & 0.0   & 0.1   &$10^{-2}$  &0.6 &1.0  &1.0  &0.0    &0.1 &1.0  &0.6    &0.1    \\
DGL-5     &$10^{-2}$  & 0.0   & 0.1   &$10^{-2}$  &0.2 &0.1  &1.0  &0.0    &1.0 &1.0  &0.2    &1.0    \\
\bottomrule
\end{tabular}

\vspace{0.35em}
\end{table}

\paragraph{Solvers and subproblems.}
All problem-specific optimization subroutines (LP/ILP relaxations, dual computations, and feasibility checks) use the same solver family and default options as in~\cite{shah_decision-focused_2022}. Concretely, the DGL and SPO+ variants call \textsc{Gurobi} via its Python API, while QPTL is implemented in \textsc{CVXPY}. This keeps solver-side effects fixed so that observed differences are attributable to the learning objective rather than solver tuning.

\paragraph{Warm start for DFL methods.}
We warm start \emph{every} decision-focused learning (DFL) variant from a common model trained with mean squared error (MSE). Concretely, we first train the predictor with the MSE loss only; we then initialize each DFL run from these MSE weights and proceed with the corresponding decision-focused objective. This choice serves two purposes. First, it isolates the effect of the \emph{loss} and \emph{dual-update policy}: all methods begin from the same functional basin, so downstream differences are not confounded by arbitrary random initializations. Second, it improves optimization stability for DFL objectives, which are known to be harder to train from scratch due to non-convex couplings between predictions and decisions; starting from a reasonable predictor reduces burn-in and variance.

\paragraph{Why not average over random initializations?}
Averaging over many random initializations primarily estimates sensitivity to the initial weights. Our warm-start protocol \emph{controls} for this factor by fixing a shared initialization that is already competitive under a standard surrogate (MSE). In pilot runs we observed that, once warm-started, between-seed variation in decision metrics is small relative to the method-to-method gaps we report. Reporting a single, controlled warm start therefore (i) focuses the comparison on the learning objective itself, (ii) reduces confounding from optimization noise, and (iii) keeps compute comparable across methods and tasks. For reproducibility, we fix library RNG seeds and release scripts that regenerate the warm-start checkpoint and all downstream results from scratch.

\section{Additional Experimental Results}
\subsection{Many-to-One Matching}
Because our many-to-one Matching setting is newly proposed and not part of prior DFL benchmarks, we also report per-seed robustness to the stochastic data-generation pipeline. Concretely, we regenerate train/val/test splits with three global seeds (10, 25, 95) and retrain every method under the same protocol as in the main experiments, including the MSE warm start and fixed training RNG seeds. We intentionally \emph{do not average curves across these seeds}: each seed induces a different dataset and thus a distinct experimental setting, so averaging would conflate heterogeneous tasks and obscure method-level differences. Instead, we present the per-seed curves and observe consistent qualitative trends across all three seeds: DGL variants reach low regret much earlier than SPO+ and QPTL; light refresh policies (e.g., DGL-1 or DGL-auto) typically improve over DGL-none at modest extra cost; and while QPTL can sometimes match or slightly surpass the best final regret, it does so at substantially higher wall-clock time. See Fig.~\ref{fig:matching_seeds}.

\begin{figure}[h!]
    \centering
    \includegraphics[width=0.8\linewidth]{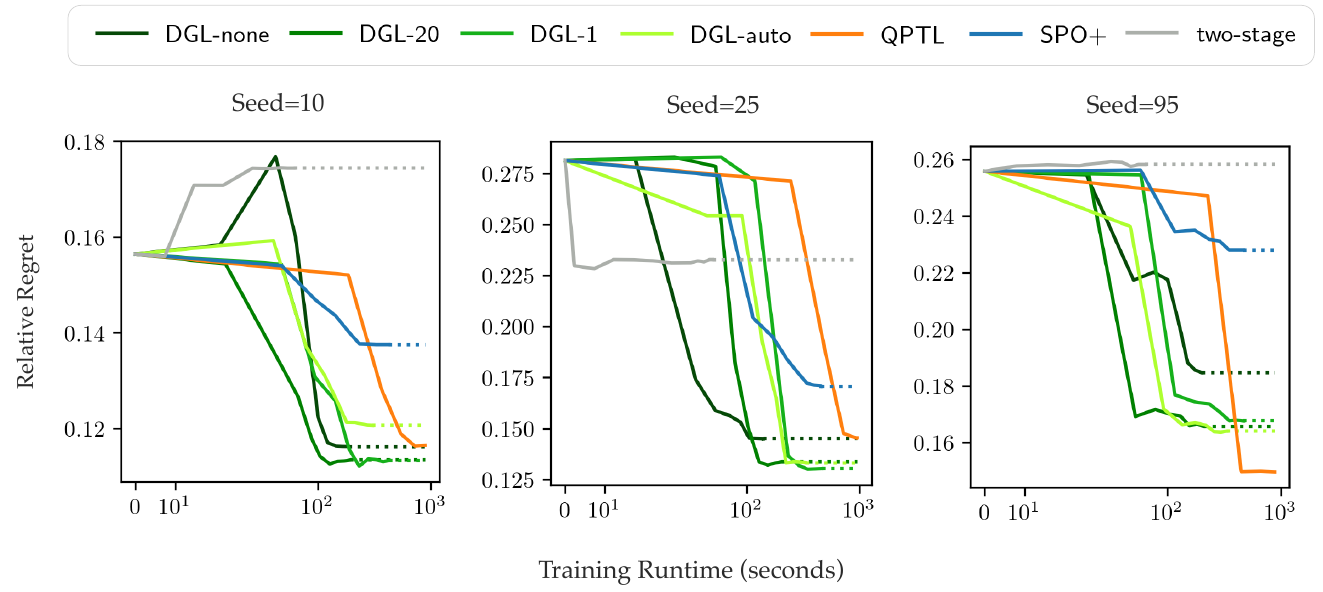}
    \caption{\textbf{Additional per-seed results for the Many-to-One Matching setting.} Each panel corresponds to a distinct dataset generated with a different global seed.}
    \label{fig:matching_seeds}
\end{figure}

\end{document}